\documentclass[10pt,twocolumn,letterpaper]{article}

\usepackage{xcolor}
\usepackage{cvpr} 
\definecolor{iccvblue}{rgb}{0.21,0.49,0.74}
\definecolor{iccvred}{rgb}{0.8,0.0,0.0} 
\usepackage[pagebackref,breaklinks,colorlinks]{hyperref}
\hypersetup{
    linkcolor=blue,  
    citecolor=blue,  
    filecolor=iccvblue,  
    urlcolor=iccvblue,   
    allcolors=blue   
}

\AtBeginDocument{
    \hypersetup{
        linkcolor=red 
    }
}

\usepackage{pifont}
\usepackage{booktabs}
\usepackage{array}
\usepackage{graphicx}
\usepackage{amsmath}
\usepackage{amssymb}
\usepackage{booktabs}
\usepackage{amsfonts,amssymb}
\usepackage{algorithm}
\usepackage{algpseudocode}
\usepackage{multirow}
\usepackage{float}
\usepackage{caption}
\usepackage{subcaption}
\usepackage{fontawesome}
\usepackage{tabularx,caption}
\usepackage[rightcaption]{sidecap}
\usepackage{colortbl}
\usepackage{psfrag}
\usepackage[percent]{overpic}
\usepackage{url}
\usepackage[accsupp]{axessibility}
\usepackage{tablefootnote}
\usepackage{caption} %
\newcommand{\tabincell}[2]{\begin{tabular}{@{}#1@{}}#2\end{tabular}}
\newcommand{\red}[1]{\textcolor{red}{#1}}
\newcommand{\black}[1]{\textcolor{black}{#1}}
\newcommand{\blue}[1]{\textcolor{blue}{#1}}
\def\ie{\textit{i.e., }}
\definecolor{mygray}{gray}{.9}
\definecolor{lightblue}{RGB}{220, 230, 255} 

\usepackage{threeparttable}

\usepackage{tikz} 

\newcommand{\bigcircle}[1]{\tikz[baseline=(char.base)]{
    \node[shape=circle, draw, fill=black!75, text=white, inner sep=0.4pt, minimum size=4pt] (char) {\textbf{#1}};}}

\newcommand{\cmark}{\ding{51}}%
\newcommand{\xmark}{\ding{55}}%

\def\httilde{\mbox{\tt\raisebox{-.5ex}{\symbol{126}}}}

\begin{document}

\title{From Noise to Meaning: Meaningful Secret Sharing with Tamper Detection for Facial Recognition}

\author{Ajnas Muhammed\textsuperscript{1}, Iurii Medvedev \textsuperscript{1}, Nuno Gonçalves\textsuperscript{1} \\
\textsuperscript{1}{Institute of Systems and Robotics, University of Coimbra, Coimbra, Portugal, 3030-290}\\
{\tt\small \{ajnas.muhammed, iurii.medvedev\}@isr.uc.pt, nunogon@deec.uc.pt}\\
}

\twocolumn[{%
\renewcommand\twocolumn[1][]{#1}%
\maketitle%
}]
\maketitle
\thispagestyle{empty}

\begin{abstract}
 Popularity of AI-based face recognition system directly demands protection of sensitive biometric data used for training. Visual secret sharing is an interesting idea, as it splits facial images into secret shares that look random and spread across many institutions. However, these shares look like noise and can easily spark suspicion and recognized as encrypted content. This makes them open to targeted collection and harvest-now-decrypt-later attacks. Additionally, visual secret sharing does not detect tampering, allowing attackers to modify shares and threaten the integrity of reconstruction. In this paper, we introduce a new method that turns distracting noise-like secret shares into visually appealing cover images with additional cryptographic tamper detection. The proposed technique works with visual secret sharing and introduces cover images to embed the shares using adaptive least significant bit steganography. Here, cover images with perceptual transparency are used to store secret shares while guaranteeing complete privacy. A two layer authentication using strong digital watermarking and cryptographic hashing is used to protect the integrity of shares. The proposed technique shows high resilience in stopping bit-flipping, cropping, and substitution attacks. Extensive experiments on multiple public face datasets show that the technique shows better FR accuracy, while eliminating share conspicuousness and guaranteeing integrity. The proposed framework sets a new standard for protecting facial data in such a way that privacy, security, and integrity are protected.
\end{abstract}


\section{Introduction}
\label{sec:intro}

Face Recognition System (FRS) is an important part of modern security applications ranging from unlocking smartphones~\cite{patel2016secure} to automating border control~\cite{hidayat2024face}. Face images are one of the most significant private personal data and are essential to design an efficient FRS. The security of face data is a critical aspect, as the faces cannot be changed once compromised~\cite{wang2025privacy}. Due to this, regulatory frameworks worldwide have responded to these concerns. The European Union-General Data Protection Regulation (EU-GDPR)~\cite{voigt2017eu}, the California Consumer Privacy Act (CCPA)~\cite{pardau2018california}, and new AI governance frameworks~\cite{ghosh2025artificial}, all set strict rules for collecting, storing and processing face biometric data. Right-To-Be-Forgotten (RTBF) is an important right in article 17 of EU-GDPR, which allows users to permanently remove their personal data~\cite{yaish2019forget}. Regular FR frameworks that depend on centralized storage cannot meet these requirements due to issues such as data replication and lack of data management~\cite{muhammed2025voidface}.

Visual Secret Sharing is a visual data protection mechanism introduced by Naor and Shamir~\cite{naor1994visual}. VSS can be used efficiently as a cryptographic base for privacy-preserving FRS. VSS allows distributed storage without a central point of failure by splitting face images into multiple random shares spread across different institutions. Recently, VOIDFace framework~\cite{voidface} successfully combined VSS with patch-based training with shares of facial patches (eyebrows, eyes, nose and mouth) instead of full face images. Even though conventional VSS has multiple advantages both in terms of privacy and security, two major limitations in real world FRS applicability are,

\begin{enumerate}
	\item Share identifiability: Traditional VSS generates meaningless shares that look like random noise~\cite{muhammed2021novel}. Due to this, these shares cannot be statistically isolated from encrypted content and immediately recognized as cryptographic artifacts. Due to this, attackers can easily find and target shares for harvest-now-decrypt-later attacks~\cite{olutimehin2025future}. The existence of noise-like data makes the presence of VSS systems obvious. In sensitive areas such as biometrics and healthcare, the storage of noise-like data causes suspicion. 

	\item No tamper detection: Traditional VSS protects privacy, but not integrity. An attacker with access can manipulate, delete, or corrupt VSS shares without being caught. This can lead to faulty reconstruction and affect performance, poison training data, or add backdoor attacks~\cite{le2024comprehensive}. Bit rot or storage errors can cause gradual corruption that can make reconstruction less accurate.
\end{enumerate}

Due to these problems, we introduce a novel privacy preserving FRS framework using VSS share generation that solves these problems. The main contributions of this framework are as follows,

\begin{enumerate}
	\item Meaningful Share Generation. The proposed technique introduces a reversible steganographic embedding method that transforms conspicuous noise-like secret shares into visually meaningful cover images. Encrypted face patches are hidden in innocuous cover images using adaptive least significant bit substitution with perceptual optimization. To our knowledge, this is the first application of adaptive LSB steganography to VSS-based facial recognition to defend harvest-now-decrypt-later attacks.

	\item Dual-Layer Tamper Detection. The technique proposes an integrity verification system that combines cryptographic hashing (SHA-256) with robust digital watermarking (DWT spread-spectrum). This defense-in-depth mechanism detects both bit-level modifications and structural tampering (cropping, substitution, compression), achieving $>$99.9\% detection rates with false positives below 0.15\%.
\end{enumerate}

The paper is organized as follows. Section \ref{sec:rel} discusses related works. The proposed framework is explained in section \ref{sec:pm}. The experimental setups and results are explained in Section \ref{sec:ex}. Section \ref{sec:co} concludes the work.

\section{Related works}
\label{sec:rel}

\textbf{VSS and its extensions}: VSS is a cryptographic technique that encrypts visual data into $n$ random looking shares such that $k$ shares reveal the secret when superimposed, known as $(k,n)$VSS. In a $(k,n)$VSS, any collection of shares $ < k$ does not reveal the secret. Traditional VSS scheme are designed for binary secret images and halftoning is used in case of gray and color images~\cite{muhammed2021novel}. However, recent works extended VSS to multiple domains with XOR operation and lossless reconstruction~\cite{nujumudeen2025lightweight}.

Multiple Image Secret Sharing (MISS) is a VSS scheme in which more than one secret image is encrypted within the same access structure. Sasaki and Watanabe~\cite{sasaki2017visual} introduced an MISS access structure by a qualified and forbidden set along with monotonicity and uniqueness. VOIDFace framework~\cite{voidface} used another MISS for face patches using XOR operation to create authentication share and private share, and improves security and privacy. Even with these improvements, the regular VSS shares still stand out. The random looks of VSS shares make them identifiable as encrypted content (ciphertext) and make a key practical deployment limitation~\cite{chen2025essential}.

\textbf{Steganography and Meaningful Shares}: An efficient way to deal with share identifiability is to use steganography. Steganography hides secret information in a harmless cover image. Dworetzky et al. introduced matrix embedding methods with minimum distortion and maximum embedding capacity~\cite{dworetzky2024improving}. Meaningful VSS schemes evolve when steganography and VSS work together. Lin and Tsai introduced techniques to hide VSS shares in cover images~\cite{lin2022large}. Another halftone-based meaningful VSS scheme is introduced in~\cite{liu2026meaningful}. Embedding capacity and reconstruction quality are the main limitations of these methods.

\textbf{Tamper detection}: In VSS, detecting tampering is a less explored area. Researchers added authentication codes to VSS shares, but with high space complexity~\cite{mohammed2024tamper, muhammed2023secure}. Bhardwaj et al.~\cite{bhardwaj2024enhancing} introduced a weak watermarking technique for VSS for tamper detection, but failed in tamper location identification and recovery. Function-correcting codes~\cite{premlal2025function} are a theoretical step towards error correction for computational tasks. Here, specific functions are used for the correct evaluation of the decoded data.

\textbf{Privacy preserving face recognition}: Privacy preserving face recognition consists of different techniques. Homomorphic encryption enables operations on encrypted data, but is computationally complex~\cite{song2025privacy}. Federated learning spreads training across multiple devices, but gradient leakage can still occur~\cite{muhammed2025federated}. Differential privacy gives mathematical guaranties, but makes accuracy degradation~\cite{he2024diff}. Secure Multi-Party Computation lets multiple parties work together on computations, but requires complicated protocols~\cite{elfares2024privateyes}.

VOIDFace framework \cite{voidface} is the first work to combine VSS with patch-based FR training, showing competitive accuracy (97.8\% in LFW) with only secret shares. But VOIDFace's obvious noise shares and lack of integrity verification leave a significant security gap that the proposed technique addresses.

\section{Proposed method}
\label{sec:pm}
\begin{figure*}[!t]
	\centering
	\includegraphics[width=0.95\linewidth]{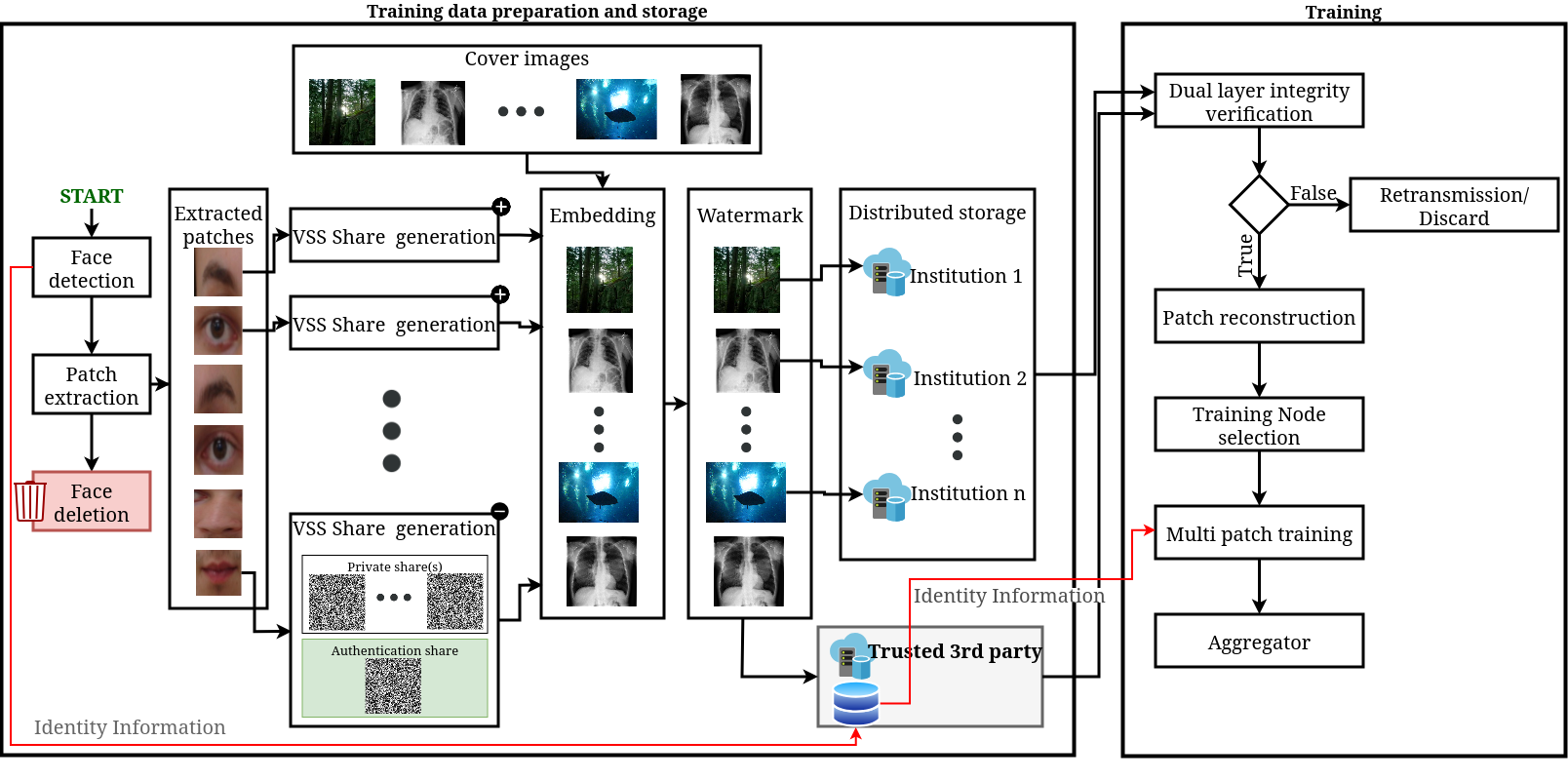}
	\caption{Overall architecture of the proposed technique showing patch extraction, meaningful share generation with watermarking, distributed storage, and verification based reconstruction, and training.}
	\label{fig:proposed}
\end{figure*}
\subsection{System Overview}
The proposed method operates within VSS-based privacy preserving FR framework-VOIDFace, extending on meaningful share generation and tamper detection capabilities. This is a crucial limitation for the real time implementation of VOIDFace. 

\textbf{Assumptions}: Similar to VOIDFace, the proposed technique operates under three assumptions:

\begin{enumerate}
	\item Frontal Faces Only: The framework only processes frontal face images, which gives clear, unobstructed views needed for efficient patch extraction.

	\item Trusted third party processing: The framework starts with face acquisition, patch extraction, and share generation, all accomplished by a trusted third party. Trusted third party also removes the original images after patch reconstruction.
	
	\item Distributed Storage: $N$ separate storage facilities and $N_p$ training workstations are required to store the final shares and training, making this resilient to single-point attacks.
\end{enumerate}

Figure \ref{fig:proposed} illustrates the proposed method's system design. The system first extracts ${\rm N_p}$ patches from each face image. The original face image is permanently deleted after patch extraction. Each patch is then encrypted as private share(s) and authentication share. The authentication share is same for all patches extracted from a particular face, whereas the private shares are different for each patch. Now, instead of storing this conspicuous noise like share, the system embeds them into visually meaningful cover images. These meaningful embedded images are watermarked and distributed to different participating institutions. When training is requested, patches are reconstructed from the meaningful shares. Before reconstitution, each share undergoes dual layer integrity verification to detect any tampering. After verification, the patches are reconstructed and passed on to ${\rm N_p}$ patch training networks, and the output features are aggregated to generate the global face embeddings for recognition.

\subsubsection{Training data protection}
The proposed FR framework starts with training data preparation and storage. The first phase of training data preparation and storage includes patch extraction and VSS share generation. Given a frontal face image ${\rm F}$, extraction of ${\rm N_{p}}$ privacy preserving face patches is carried out. After patch extraction, each ${\rm P_{i}}$ is resized to equal size and ${\rm F}$ is permanently deleted and never stored or transmitted at any stage of the framework.

After patch extraction, each ${\rm P_{i}}$ is passed to the VSS share generation module. As each user has ${\rm N_{p}}$ face patches, a multiple patch secret sharing mechanism is used to include all ${\rm N_{p}}$ patches. Initially, a random matrix of the patch size is generated using the cryptographic random number generator function. This random matrix is designated as the user's authentication share ${\rm AS}$. Similarly, for each ${\rm P_{i}}$, the patch encryption (share generation) ${\rm E_{i}}$ is calculated as Eq. \ref{eq:1}.
\begin{equation}
	E_i = P_i \oplus AS
	\label{eq:1}
\end{equation}
Next is the generation of meaningful VSS shares. In this stage, meaningless noise-like shares are transformed into visually meaningful shares. This process starts with the selection of cover images. The cover images are selected based on the deployment contexts such as medical images (chest radiograph) for healthcare, and landscape for general use. The authentication cover image ${\rm C_{auth}}$ is a single cover image, and private cover images ${\rm C_i}$ are multiple distinct cover images. Cover images are selected in such a way that the selected image must have sufficient embedding capacity and high texture complexity.

After selecting the cover image, ${\rm AS}$ and ${\rm E_i}$ are embedded in ${\rm C_{auth}}$ and corresponding ${\rm C_i}$, respectively. The embedding is performed using adaptive least significant bit (LSB) embedding~\cite{hussain2021enhanced}. The adaptive LSB embedding performs the embedding following adaptive LSB substitution based on the complexity of the local texture.

Let ${\rm \sigma(x,y)}$ be the local standard deviation in a $3 \times 3$ neighborhood, the embedding depth ${\rm d(x,y)}$ is calculated as Eq. \ref{eq:ed}. Here, the embedding depth ${\rm d(x,y)}$ shows the number of LSB modified to embed the secret data. Here, the value of ${\rm d(x,y)}$ is restricted to 1-3 bits per channel.
\begin{equation} \label{eq:ed}	
	d(x,y) = min(3, max(1, \lfloor \frac{\sigma(x,y)}{\sigma^{max}}.3 \rfloor))
\end{equation}
For each pixel ${\rm (x,y,c)}$, the embedding process extracts ${\rm d(x,y)}$ bits from ${\rm E_i}$, replace ${\rm d(x,y)}$ LSBs of cover image pixels ${\rm C_i(x,y,c)}$. The embedding depth is stored in depth map ${\rm D_i}$. The result of this embedding generates a meaningful share ${\rm M_i}$, which is visually indistinguishable from corresponding ${\rm C_i}$. Similar to ${\rm P_{i}}$, ${\rm AS}$ is also embedded into  ${\rm C_{auth}}$ using the same technique. 

The proposed technique incorporates tamper detection as another contribution. Tamper detection ensures share integrity by detecting malicious and accidental tampering before patch reconstruction and FR training. Tamper detection is achieved using the dual layer authentication, i.e., cryptographic hash and robust watermark. Cryptographic hash provides bit-level integrity, whereas robust watermark provides structural integrity. These layers operate sequentially and are complement to each other, such as the former one fails if attacker replace entire share, whereas later fails against bit-flip attacks that preserve the structure. 

For each embedded ${\rm C_i}$ (or $C_{auth}$), the cryptographic hash is computed as:
\begin{equation}
	\label{eq:ch}
	h_i = SHA256(C_i||salt_i)
\end{equation}
where, ${\rm SHA256}$, $||$ and ${\rm salt_i}$ represent the cryptographic hash function, concatenation operation, and salt derived from the shares metadata (share ID and time stamp). Salt ensure that the same secret embedded in different shares produce different and pre-computation attack is unfeasible. 

After the hash generation, they are embedded in reserved pixel positions deterministically chosen from the shares. These locations are spread to resist localized corrections. Each hash bits are replaced in the most significant LSB of the pixel embedding depth to maximize the sensitivity.

During the second layer, robust watermark, shares are converted into frequency domain using spread-spectrum modulation, and digital watermark is embedded inside. The watermark is generated using 176-bit payload defined in Eq. \ref{eq:wp}, including 16-bit share ID, 32-bit timestamp, 96-bit truncated hash (first 96 bit of cryptographic hash) and 32-bit checksum. The watermark is designed to survive benign operations such as JPEG compression and mild scaling, whereas it is vulnerable to malicious tampering, such as cropping, substitution, and perturbation.
\begin{equation}
	\label{eq:wp}	
	W=ID||T||h_{96}||CRC32(ID||T||h_{96})
\end{equation}
The watermark is embedded using spread-spectrum modulation using two level Discrete Wavelet Transform (DWT). The first level decomposes the share image into four subbands ($\{LL_1, LH_1, HL_1, HH_1\}$). Now the second DWT is applied on low frequency sub-band, $LL_1$ ($\{LL_2, LH_2, HL_2, HH_2\}$). The payload $W$ is embedded into pseudo-random sequence $s$ using spread-spectrum encoding. The pseudo random sequence $s$ is calculated as,
\begin{equation}
	\label{eq:s}	
	s = \sum_{i=0}^{|W|-1} w_i . r_i
\end{equation}
where $W$ is the payload with $w_i$ showing the $i^{th}$ payload bit, and $r_i$ shows the random number generated from the seed using a master key and share's meta data.

${\rm C_{auth}}$ and ${\rm C_i}$ after LSB, hash and watermark embedding is termed as ${\rm M_{auth}}$ and ${\rm M_i}$. Now, the final MISS access structure of proposed technique is defined as the minimally refined perfect access structure as,
\begin{align}	
	\begin{split}
		\left(Q^i\right)_0 = \left\lbrace \left\lbrace M_{auth}, \ M_i \right\rbrace \right\rbrace \\
		\left(F^i\right)_0 = 2^S - \left(Q^i\right)_0
	\end{split}
\end{align}
where $\left(Q^i\right)_0$ and $\left(F^i\right)_0$ shows the minimally refined qualified and forbidden sets, respectively.

After embedding, ${\rm M_{auth}}$ and ${\rm M_i}$ are passed to the share distribution phase. In this phase, ${\rm M_{auth}}$ is stored in the trusted third part, whereas ${\rm M_i}$'s are stored in different distributed storage facilities. 

\subsection{Patch reconstruction and training}

When an FR training is initiated, the trusted third party first verifies the activation of authentication share. This facilitates the RTBF property. The user can request trusted third party to remove the authorization share both (${\rm M_{auth}}$) whenever the user wants the data to be ``forgotten". Upon successful authentication, both ${\rm M_{auth}}$ and ${\rm M_i}$ undergo integrity verification, otherwise the system request to remove all ${\rm M_i}$'s associated with the user. 

Integrity verification begins with robust watermark verification. Algorithm \ref{alg:rw} shows the watermark extraction and verification process embedded in the DWT of meaningful share. The algorithm takes the embedded meaningful share ${\rm M}$ (${\rm M_{auth}}$ or ${\rm M_{i}}$), master key ${\rm K_{m}}$, expected share identification ${\rm ID_{exp}}$, expected timestamp ${\rm T_{exp}}$ and time tolerance ${\rm Delta_{T}}$ as input. For each of the 176 payload bits, \textit{DeriveSpreadingSeq} generates a pseudo-random spreading sequence ${\rm r_{j}}$ using ${\rm K_{m}}$ and share's meta data. Later, the correlation between the DWT coefficients and ${\rm r_{j}}$ is decoded as positive and negative correlation. The \textit{ParsePayload} function is used to interpret the structured data from the 176-bit watermark bitstream $(\mathcal{W})$. The algorithm verifies the CRC32 checksum and checks that the extracted share identifier and timestamp match the expected values within a tolerance ${\rm Delta_{T}}$, and returns ``True" if all checks are passed.

\begin{algorithm}[!h]
	\caption{Robust watermark verification}
	\label{alg:rw}
	\begin{algorithmic}[1]
		\Require $M$, $K_m$, $ID_{exp}$, $T_{exp}$, $\Delta_T$
		\Ensure Verification status, $\mathcal{W}$
		\State $\text{DWT2}(M) = (LL_2, LH_2, HL_2, HH_2)$ 
		\State $L = \text{length}(LH_2)$
		\State $\text{bits} = []$
		
		\For{$j = 0$ \textbf{to} $175$} 
		\State $\mathbf{r}_j = \text{DeriveSpreadingSeq}(K_m, ID_{exp}, T_{exp}, j, L)$
		\State $\text{corr} = \frac{1}{L} \sum_{i=0}^{L-1} \frac{LH_2[i] + HL_2[i]}{2} \cdot \mathbf{r}_j[i]$  
		\If{$\text{corr} > 0$}
		\State $\text{bits}[j] = 1$
		\Else
		\State $\text{bits}[j] = 0$
		\EndIf
		\EndFor
		\State $\mathcal{W} = \text{bits}$
		\State $ID_{ext}, T_{ext}, h_{96}, \text{crc} = \text{ParsePayload}(\mathcal{W})$
		\If{$\text{CRC32}(ID_{ext} \parallel T_{ext} \parallel h_{96}) \neq \text{crc}$}
		\State \Return $(\text{False}, \emptyset)$
		\EndIf
		\If{$ID_{ext} \neq ID_{exp}$ \textbf{or} $|T_{ext} - T_{exp}| > \Delta_T$}
		\State \Return $(\text{False}, \emptyset)$
		\EndIf
		\State \Return $(\text{True}, \mathcal{W})$
	\end{algorithmic}
\end{algorithm}

Once Algorithm \ref{alg:rw} returns ``True", the next step verifies the bit level integrity of the embedded secret data using Cryptographic Hash verification (Algorithm \ref{alg:chv}). Algorithm \ref{alg:chv} takes embedded meaningful share ${\rm M}$, share identification ${\rm ID}$, timestamp ${\rm T}$ and master key ${\rm K_{m}}$ as input. The verification first generates a salt $\sigma$ using HMAC-based Key Derivation Function (HKDF) with ${\rm K_{m}}$ and share meta data, and then derives $256$ embedding positions ${\mathcal{P}}$. The depth map ${\rm D}$ is extracted from the share's reserved region. For each ${\mathcal{P}}$, embedded hash bits are identified as stored in ${\rm h_{stored}}$. Simultaneously, full embedded secret data is extracted using adaptive LSB extraction (Algorithm \ref{alg:ale}) guided by the depth map. Algorithm \ref{alg:ale} retrieves the embedding depth $d$ from depth map, extract $d$ least significant bits by applying bit-mask $((1 \ll d) - 1)$, and stored into corresponding positions of the secret data. The final verification is performed by computing hash and comparing it with  ${\rm h_{stored}}$.

\begin{algorithm}[!h]
	\caption{Cryptographic Hash Verification}
	\label{alg:chv}
	\begin{algorithmic}[1]
		\Require $M$, $ID$, $T$, $K_m$, image dimensions $H, W$
		\Ensure Verification status
		\State $\sigma = \text{DeriveSalt}(K_m, ID, T)$
		\State $\mathcal{P} = \text{DeriveHashPositions}(K_m, ID, T, H, W)$
		\State $D = \text{ExtractDepthMap}(M)$ 
		\State $H, W, C = \text{dimensions}(M)$
		\State $\text{stored\_bits} = []$
		\For{$(x,y,c) \in \mathcal{P}$}
		\State $d = D[x,y,c]$
		\If{$d \ge 1$}
		\State $p = M[x,y,c]$
		\State $b = (p \gg (d-1)) \& 1$
		\State $\text{stored\_bits}.\text{append}(b)$
		\Else
		\State \Return $\text{False}$
		\EndIf
		\EndFor
		\State $h_{stored} = \text{BitsToBytes}(\text{stored\_bits})$
		\State $S = \text{ExtractAdaptiveLSB}(M, D)$ \Comment{Algorithm \ref{alg:ale}}
		\State $h_{curr} = \text{SHA256}(S \parallel \sigma)$
		\If{$h_{curr} = h_{stored}$}
		\State \Return $\text{True}$
		\Else
		\State \Return $\text{False}$
		\EndIf
	\end{algorithmic}
\end{algorithm}

In patch reconstruction, if any of the verification (Algorithm \ref{alg:rw} or Algorithm \ref{alg:chv}) fails, the proposed technique rejects the data for reconstruction and training (as shown in Table \ref{tab:dm}). This ensures that only authentic shares proceeds to patch reconstruction and training. The proposed technique thus provides as a defense‑in‑depth mechanism, in which watermark layer capture structural tampering, where as hash catches bit level tampering. An attack that must compromise both layers to go undetected is computationally infeasible. The dual layer integrity verification ensures that only authentic, tamper-free shares proceeds to patch reconstruction and training, thus preventing from data poisoning attacks. The additional supporting algorithms are given in supplementary material.

\begin{algorithm}[!h]
	\caption{Adaptive LSB Extraction}
	\label{alg:ale}
	\renewcommand{\arraystretch}{1}
	\centering
	\begin{algorithmic}[1]
		\Require $M$, $D$
		\Ensure Secret data $S$
		
		\State $H, W, C = \text{dimensions}(M)$
		\State $S = \text{array of zeros}(H, W, C)$
		
		\For{$x = 0$ \textbf{to} $H-1$}
		\For{$y = 0$ \textbf{to} $W-1$}
		\For{$c = 0$ \textbf{to} $C-1$}
		\State $d = D[x,y,c]$
		
		\If{$d \ge 1$}
		\State $p = M[x,y,c]$
		\State $S[x,y,c] = p AND ((1 \ll d) - 1)$
		\Else
		\State $S[x,y,c] = 0$
		\EndIf
		\EndFor
		\EndFor
		\EndFor
		
		\State \Return $S$
		
	\end{algorithmic}
\end{algorithm}

\begin{table*}[!h]
	\centering
	\caption{Decision matrix for dual-layer tamper detection}
	\label{tab:dm}
	\begin{tabular}{cccl}
			\toprule
			\textbf{\begin{tabular}[c]{@{}c@{}}Watermark\\ Status\end{tabular}} & \textbf{\begin{tabular}[c]{@{}c@{}}Hash \\ status\end{tabular}} & \textbf{Decision} & \multicolumn{1}{c}{\textbf{Action}}                                               \\ \midrule
			Valid                                                               & Matches                                                         & Accept            & \begin{tabular}[c]{@{}l@{}}Proceed to reconstruction\end{tabular}              \\
			Valid                                                               & Mismatches                                                      & Reject            & \begin{tabular}[c]{@{}l@{}}Hash tampered; request retransmission\end{tabular}  \\
			Corrupted                                                           & Any                                                             & Reject            & \begin{tabular}[c]{@{}l@{}}Watermark tampered; log and retransmit\end{tabular} \\
			Missing                                                             & Any                                                             & Reject            & \begin{tabular}[c]{@{}l@{}}Severe tampering; forensic flag\end{tabular}        \\ \bottomrule
	\end{tabular}
\end{table*}

After the dual verification and extraction of embedded data, the original patch is reconstructed as,
\begin{align}	
	\overline{P_i} = \overline{AS} \oplus \overline{E_i}
\end{align}  
where $\overline{P_i}$, $\overline{AS}$, and $\overline{E_i}$ are the $i^{th}$ reconstructed patch, extracted data from corresponding ${\rm M_{auth}}$ , and $i^{th}$ private share extracted from ${\rm M_{i}}$, respectively. 

After the patch reconstruction, the reconstructed patches are proceeded to multi-patch training network (MPTN). MPTN consists of $N_p$ patch training networks and an aggregator. Each patch training network used a MobileNet backbone to extract 512 dimensional feature embedding from the corresponding patch. Aggregator consist of fully connected layer that concatenate and consolidate all $N_p$ patch feature embeddings into a single 512 dimensional global face embeddings.

\section{Experimental results and discussion}
\label{sec:ex}

\subsection{Experimental setup}
\textbf{Dataset used}: The proposed technique used VGGFace2 dataset~\cite{cao2018vggface2} for the evaluation. As the proposed technique relies only on front facing images, a quality driven filtering technique~\cite{medvedev2023improving} with False Rejection Rate (FRR) of $0.05$ is adopted for front facing image selection. The selection of low FRR reduces the exclusion of genuine images by focusing on clear front facing images. Table \ref{tab:dataset} shows the further details of VGGFace2 dataset with adopted filtering. In the proposed method, the cover images are selected from three different datasets, including Places365\footnote{\url{https://www.kaggle.com/datasets/bobaaayoung/place365}}, ImageNet~\cite{deng2009imagenet}, chest Xray\footnote{\url{https://www.kaggle.com/datasets/animeshshedge/chest-x-rays-of-14-common-disease?select=XRays}}.

\begin{table}[!htb]
	\centering
	
	\caption{VGGFace2 dataset details before and after quality driven filtering.}
	\label{tab:dataset}
	\begin{tabular}{cccc}
		\hline
		\textbf{Filtering}                                            & \textbf{\begin{tabular}[c]{@{}c@{}}\# images\end{tabular}} & \textbf{\begin{tabular}[c]{@{}c@{}}\# Classes\end{tabular}} & \textbf{\begin{tabular}[c]{@{}c@{}}\# Average Images \\per Class\end{tabular}} \\ \hline
		
		\begin{tabular}[c]{@{}c@{}}None\end{tabular} & 3074k                                                               & 8631                                                                  & 356                                                                 \\
		\begin{tabular}[c]{@{}c@{}}FRR=0.05\end{tabular} & 1158K                                                               & 8628                                                                  & 134                                                                 \\
		\hline
	\end{tabular}
\end{table}

\textbf{Preprocessing}: During the proposed technique preprocessing, a user face is detected and segmented using Histogram of Oriented Gradients (HOG)-based (or MTCNN) method combined with linear classifier. For the experimentation, six patches (${\rm N_p}$=
6) are used including both eyes, both eyebrows, nose, and mouth. During preprocessing, ${\rm N_p}$ patches of the frontal face images are extracted using a pre-trained model\footnote{\url{https://www.kaggle.com/code/zeyadkhalid/face-landmarks-detection-and-alignment-dlib}}.

\textbf{Models}: The proposed method's training uses Patch Training Networks (PTN) and Aggregator. PTN include MobileNet backbone with an output dense patch feature layer. Patch features are concatenated and fed into the Aggregator. PTN is optimized using SGD optimizer with momentum 0.5, initial learning rate 0.01, and Cosine Annealing schedule ($\eta_{\mathrm{max}}$ = 0.01, $\eta_{\mathrm{min}}$ = $1\mathrm{e}{-7}$) over 20 epochs. Training processes 96$\times$96 RGB images of all six facial patches with batch size of 10. Images are normalized by subtracting [0.5, 0.5, 0.5] and scaling $\frac{1}{255}$. We train categorical cross-entropy loss with 1.0 weights for all patches in 8628 classes. Important architecture parameters include a width multiplier of 1.4 and an ArcFace margin of 0.5. The Aggregator uses a simple fully connected layer to combine patch feature outputs into a global feature vector. We also examined a modified PTN that added patch-level classification supervision, turning the problem into a multitask learning approach where PTN and Aggregator perform the same classification task. This improves patch-level supervision and patch selection and augmentation flexibility. 

\subsection{Training data analysis}
\begin{figure*}
	\centering
	\includegraphics[width=0.75\linewidth]{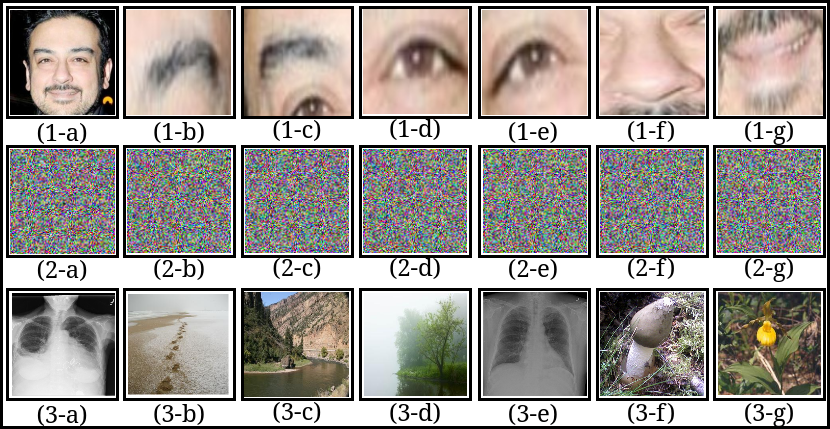}
	\caption{Sample meaningful shares generated. 1-a: A face image, 1-b to 1-g: Six face patches, 2-a: Authentication share, 2-b to 2-g: Private shares, 3-a to 3-g: Cover images after share embedding and watermarking}
	\label{fig:sample}
\end{figure*}
The proposed method gives participating institutions visually meaningful shares instead of noise-like secret shares. In contrast to traditional VSS, meaningful shares look like normal images like medical X-rays, landscape photos, or abstract patterns, depending on the deployment context (see Figure \ref{fig:sample}). This introduces three major benefits for storage institutions such as stealth and regulatory acceptance: By resembling normal data, such as medical images, shares are not detected by security audits or adversaries, reducing the risk of targeted collection and harvest-now-decrypt-later attacks. Second, simplified data management: Institutions can store these shares using their image storage infrastructure without encryption detection or handling. Third, operational resilience: Even if an institution suffers a data breach, these meaningful share contains no information about the original facial data. Without authentication share (stored exclusively by a trusted third party) and master key, an attacker cannot reconstruct any facial patch. Institutions can maintain legal compliance, public trust, and collaborative FR training by hosting privacy-preserving face data without becoming high-value targets.

Traditional FRS replicates the same facial datasets across multiple institutions for training, resulting in redundant storage, inconsistent data management, and increased privacy risks, as each copy is an attack point. The proposed technique solves this problem by distributing only meaningful private shares across institutions, which is unique and useless. Since no two institutions store the same data, data replication is avoided and the storage footprint is much lower than with full dataset. The framework also integrates the RTBF property using ${\rm M_{auth}}$. The trusted third party deletes ${\rm M_{auth}}$ when a user withdraws consent. Without this share, the distributed private shares cannot be reconstructed into facial patches, rendering them useless for future training. Orphan private shares are deleted by periodic checking.

\subsection{Security analysis}
Table \ref{tab:aa} shows the attack analysis of hash layer, watermark layer, and combined layer on different attacks. Hash layer detects modifications on embedded data and watermark layer detects structural modification. In the proposed technique, both layers must be compromised to perform an undetected attack, which is cryptographically infeasible as it requires breaking SHA-256 and the master key ${\rm K_m}$. If one layer is compromised, the other layer detects attacks and requests for retransmission if necessary.
\begin{table*}[!h]
	\caption{Attack analysis of hash layer, watermark layer, and combined detection} \label{tab:aa}
	\centering
		\begin{tabular}{lccc} \\ \toprule
			\multicolumn{1}{c}{\textbf{Attack Type}} & \textbf{Hash Layer} & \textbf{Watermark Layer} & \textbf{Combined Detection} \\ \midrule
			Single-bit flip in secret                 & \cmark                   & \xmark                       & \cmark                           \\
			Multiple-bit flip (random)                & \cmark                  & \cmark                       & \cmark                          \\
			Cropping (10\% border)                    & \xmark (may miss)        & \cmark                       & \cmark(\textgreater{}98\%)      \\
			JPEG compression (Q $\ge$ 70)                   & \cmark(bits preserved)  & \cmark                       & \cmark                          \\
			JPEG compression (Q\textless{}50)         & \xmark (bits lost)       & \cmark                       & \cmark($\ge$ 85\%)                   \\
			Scaling (90\%)                            & \xmark (alignment loss)  & \cmark                       & \cmark($\ge$ 92\%)                   \\
			Share replacement                         & \xmark (hash matches)    & \cmark                       & \cmark                          \\
			Adversarial perturbation                  & \cmark                  & \cmark                       & \cmark($>$99\%)     \\ \bottomrule
		\end{tabular}  
\end{table*}

The proposed technique exhibits an effective balance between visual integrity and embedding capacity across all six facial patches (see Table \ref{tab:eq}). While Structural Similarity Index (SSIM) scores of $0.96$–$0.97$ verify nearly comparable structural preservation to the original cover images, Peak Signal-to-Noise Ratio (PSNR) values vary from $41.7$dB to $42.5$dB, significantly over the perceptually transparent threshold of $40$dB. The average embedding capacity is between $2.6$ and $2.9$ bits per pixel (bpp), which is more than the minimum of $2$ bpp needed for full facial patch embedding. These findings verify that the proposed technique generates meaningful shares that are visually undetectable and have enough capacity for realistic privacy-preserving FR deployment.

\begin{table}[!htb]
	\centering
	
	\caption{Embedding quality analysis for all six facial patches} \label{tab:eq}
	\begin{tabular}{lccc}  \\ \toprule
		\textbf{Patch}         & \textbf{PSNR (dB)} & \textbf{SSIM} & \begin{tabular}[c]{@{}c@{}}\textbf{Embedding} \\ \textbf{Capacity (bpp)}\end{tabular} \\ \midrule
		Left Eyebrow  & 42.3      & 0.97 & 2.8                                                                 \\
		Right Eyebrow & 42.1      & 0.97 & 2.7                                                                 \\
		Left Eye      & 41.8      & 0.96 & 2.9                                                                 \\
		Right Eye     & 41.9      & 0.96 & 2.8                                                                 \\
		Nose          & 42.5      & 0.97 & 2.6                                                                 \\
		Mouth         & 41.7      & 0.96 & 2.9                                                                \\ \bottomrule
	\end{tabular}
\end{table}

\begin{table}[!h]
	\centering
	
	\caption{Model inversion attack results comparing ArcFace, VOIDFace, and Proposed technique} \label{tab:mi}
	\begin{tabular}{lcc} \toprule
		\textbf{Method}   & \multicolumn{1}{l}{\textbf{Attack Accuracy}} & \multicolumn{1}{l}{\textbf{KNN Distance}} \\ \midrule
		ArcFace~\cite{deng2019arcface}   & 82.4\%                                       & 1247.28                                   \\
		VOIDFace~\cite{voidface} & 12.1\%                                       & 2240.30                                   \\
		Proposed technique    & 11.8\%                                       & 2281.45         \\ \bottomrule                         
	\end{tabular}
\end{table}

Table \ref{tab:mi} shows the model inversion attack computed using the black-box attack simulation technique \cite{nguyen2023re}. Here, a pre-trained face.evoLve model (trained on the CelebA~\cite{liu2015deep} dataset) is used as the evaluation model and the aggregator model as the target model. While KNN distance quantifies the distance between the reconstructed samples and actual training data, a higher distance denotes stronger privacy protection, and attack accuracy measures the training data reconstruction accuracy. With a KNN distance of $1247.28$, ArcFace~\cite{deng2019arcface} achieves $82.4\%$ attack accuracy, indicating significant privacy leakage when the model memorizes certain face traits. On the other hand, both VOIDFace~\cite{voidface}, and proposed technique achieve significantly larger KNN distances of $2240.30$ and $2281.45$ while lowering attack accuracy to $12.1\%$ and $11.8\%$, respectively. As these frameworks only train on face patches instead of raw facial photos, the model is prevented from reconstruction attacks. The result confirms that the proposed technique preserves VOIDFace's strong privacy protection against reconstruction attacks. The marginal difference (12.1\% vs. 11.8\%) demonstrates that adding meaningful share generation and tamper detection does not compromise privacy.

\begin{table}[!h]
	\centering
	
	\caption{Steganalysis detection rates for meaningful shares} \label{tab:sa}
	\begin{tabular}{lcc} \toprule
		\textbf{Cover Type} & \multicolumn{1}{l}{\textbf{\begin{tabular}[c]{@{}l@{}}SRNet \\ Detection Rate\end{tabular}}} & \multicolumn{1}{l}{\textbf{\begin{tabular}[c]{@{}l@{}}XuNet \\ Detection Rate\end{tabular}}} \\ \midrule
		Medical             & 51.2\%                                                                                       & 50.8\%                                                                                       \\
		Natural             & 50.7\%                                                                                       & 51.1\%                                                                                       \\
		Abstract            & 49.8\%                                                                                       & 50.3\%     \\ \bottomrule                                                                                 
	\end{tabular}
\end{table}

The secrecy of shares is assessed in Table \ref{tab:sa} using two cutting-edge steganalysis methods, SRNet~\cite{boroumand2018deep} and XuNet~\cite{yang2024novel}. When a steganalyzer is unable to detect meaningful shares from cover images, the optimal stealthy embedding achieves a detection probability of 50\%, which is equal to random guessing. The detection rates for each of the three cover categories (Medical, Natural, and Abstract) vary from 49.8\% to 51.2\%, all within ±1.2\% of the optimal baseline. By preventing attackers from identifying which images contain embedded secrets, these results effectively counter harvest now decrypt later attacks assaults by confirming that meaningful shares are statistically indistinguishable from natural cover images.

\subsection{Performance analysis}
Here, the computation complexity of the proposed technique is analyzed using time and space complexity. Table \ref{tab:cc} shows the additional computational complexity associated with integrity verification. For each individual user, consisting of six private shares and one authentication share, the additional space overhead is 378 (54 $\times$ 7) bytes.

\begin{table}[!h]
	\centering
	
	\caption{Computational complexity analysis of dual-layer tamper detection.\textit{Platform: AMD Ryzen 5 5500, 32GB RAM, Ubuntu 24.04.4 LTS, Python 3.10.9}} \label{tab:cc}
	\begin{tabular}{llcc} \\ \toprule
		\multirow{6}{*}{\begin{tabular}[c]{@{}l@{}}Time\\ complexity\end{tabular}}   & \begin{tabular}[c]{@{}l@{}}Hash computation \\ (SHA-256)\end{tabular} & \multicolumn{2}{c}{$\sim$2 $\mu s$}   \\
		& Hash embedding                                                        & \multicolumn{2}{c}{$\sim$15 $\mu s$}  \\
		& Hash verification                                                     & \multicolumn{2}{c}{$\sim$18 $\mu s$}  \\
		& \begin{tabular}[c]{@{}l@{}}Watermark embedding \\ (DWT)\end{tabular}  & \multicolumn{2}{c}{$\sim$120 $\mu s$} \\
		& Watermark extraction                                                  & \multicolumn{2}{c}{$\sim$110 $\mu s$} \\
		& Total Verification                                                    & \multicolumn{2}{c}{$\sim$128 $\mu s$} \\ \hline
		\multirow{3}{*}{\begin{tabular}[c]{@{}l@{}}Space \\ complexity\end{tabular}} & Hash value                                                            & 256 bits         & 32 bytes      \\
		& Watermark (embedded)                                                  & 176 bits         & 22 bytes      \\
		& Total                                                                 & 432 bits         & 54 bytes    \\ \bottomrule 
	\end{tabular}
\end{table}

\begin{figure}[!htb]
	\centering
	\includegraphics[width=0.99\linewidth]{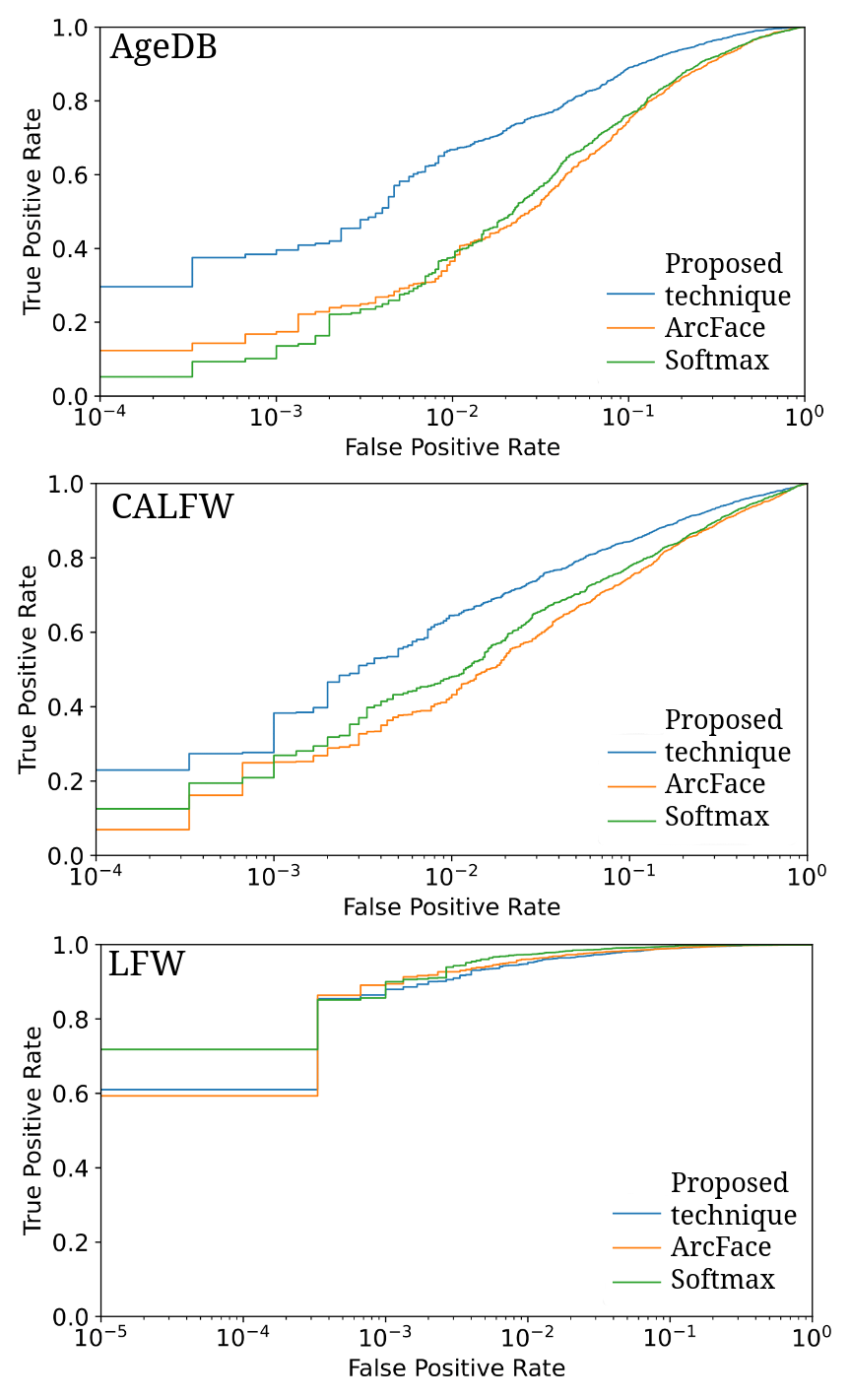}
	\caption{ROC curves of proposed technique, Arcface, and Softmax models on various benchmarks.}
	\label{fig:performance_11}
\end{figure}

We evaluated the proposed technique with traditional deep network FR training techniques and compared the performance. In particular, we used a ResNet50 backbone to assess our model against Softmax loss~\cite{liu2016large} and ArcFace~\cite{deng2019arcface}. All models were trained for 20 epochs using SGD with a gradually decreasing learning rate ($0.01$ to $0.00001$), initialized with ImageNet weights, and evaluated on LFW \cite{LFWTechUpdate}, CALFW \cite{CALFW}, and AgeDB-30 \cite{agedb30}. These benchmarks are chosen as they do not exhibit excessive pose fluctuations. Figure \ref{fig:performance_11} shows that the proposed technique performs better as compared with both Softmax and ArcFace on all benchmarks. These findings demonstrate that patch-based multi-network training with meaningful share generation and tamper detection adds privacy and integrity guarantees while maintaining recognition accuracy.

\section{Conclusion and future works}
\label{sec:co}

The proposed work provides a novel FR framework that transforms noise-like secret shares into visually meaningful cover images and provides cryptographic tamper detection. The two main drawbacks of conventional VSS techniques are share conspicuousness and lack of integrity verification. The proposed technique addresses these problems. The technique achieves statistical indistinguishability from natural images (PSNR $>$40 dB, SSIM $>$0.95) while retaining perfect secrecy guarantees by employing adaptive LSB steganography to embed noise-like shares into harmless cover images. By combining a semi-fragile DWT spread-spectrum watermark with a fragile cryptographic hash (SHA-256), the dual-layer tamper detection system achieves $>$99.9\% detection rates against bit-flipping, cropping, substitution, and compression attacks. The technique also reduces model inversion attack success rates to 11.8\%  while maintaining competitive FR accuracy. The architecture is appropriate for sensitive deployments like government identity management, border security, and healthcare since it completely maintains the RTBF property and also removes data replication.

In the future, we plan to extend our evaluation to more challenging datasets. Additionally, the framework will be expanded to include other biometrics like voice, iris, and fingerprints. This will require changing the adaptive embedding method to work with grayscale and 1D signal domains. Self-recovery methods will also be looked into so that corrupted shares can be automatically fixed when tampering is found, and reduces retransmission. Also, improvements will be made for edge devices with limited resources, like smartphones and Internet of Things (IoT) cameras using light steganography and fast DWT implementations to allow share generation and verification on the device itself.

\section*{Acknowledgment}
This project was funded by the EU’s Horizon Europe project ACHILLES under Grant Agreement No $101189689$, and from national funds through FCT - Fundação para a Ciência e a Tecnologia, under the project $UID/00048$.

{\small
\bibliographystyle{ieee}
\bibliography{ref}
}

\twocolumn[
\begin{@twocolumnfalse}
	\section*{\centering \Large{{Supplementary Material} -- ``From Noise to Meaning: Meaningful Secret Sharing with Tamper Detection for Facial Recognition''}}
\end{@twocolumnfalse}
]

\newtheorem{theorem}{Theorem} 
\newtheorem{proof}{Proof}

This supplementary material contains four sections, including the algorithms used, additional security experiments, attack scenarios and interception, and performance analysis. These experiments are conducted to reinforce the effectiveness of the proposed  framework.

\section{Additional supporting algorithms used}

\begin{algorithm}

	\caption{Salt Derivation using HKDF}
	\label{alg:salt}
	\begin{algorithmic}[1]
		\Require Master key $K_m$ (256 bits), share identifier $ID$, timestamp $T$
		\Ensure Salt $\sigma$ (128 bits)
		
		\Function{DeriveSalt}{$K_m, ID, T$}
		
		\State $\text{info} = \text{``hash\_salt"} \parallel ID \parallel T$
		\State $\sigma = \text{HKDF}(K_m, \text{info}, \text{length}=16)$
		
		\State \Return $\sigma$
		
		\EndFunction
	\end{algorithmic}
\end{algorithm}

\begin{algorithm}
	\caption{Hash Position Derivation}
	\label{alg:positions}
	\begin{algorithmic}[1]
		\Require Master key $K_m$, share identifier $ID$, timestamp $T$, image height $H$, image width $W$
		\Ensure Set of 256 positions $\mathcal{P} = \{(x_k, y_k, c_k)\}_{k=0}^{255}$
		
		\Function{DeriveHashPositions}{ $K_m, ID, T, H, W$}
		
		\State $\text{info} = \text{"hash\_positions"} \parallel ID \parallel T$
		\State $\text{seed} = \text{HKDF}(K_m, \text{info}, \text{length}=32)$
		
		\State $\text{prng} = \text{CSPRNG}(\text{seed})$ 
		
		\State $\mathcal{P} = \emptyset$
		
		\For{$k = 0$ \textbf{to} $255$}
		\State $x = \text{prng}.\text{next}() \bmod H$
		\State $y = \text{prng}.\text{next}() \bmod W$
		\State $c = \text{prng}.\text{next}() \bmod 3$
		\State $\mathcal{P} = \mathcal{P} \cup \{(x, y, c)\}$
		\EndFor
		
		\State \Return $\mathcal{P}$
		
		\EndFunction
	\end{algorithmic}
\end{algorithm}

The Algorithm \ref{alg:salt} shows the steps to produce unique 128-bit cryptographic salt $\sigma$ using HMAC-based Key Derivation Function (HKDF) on the master key ${\rm K_m}$. The algorithm first calculates an info string that combines the string ``hash\_salt", the share identifier ${\rm ID}$, and the timestamp ${\rm T}$. This deterministic construction allows the trusted third part or training institutions to reconstruct the identical salt during verification. This is used in the cryptographic hash verification, where it concatenated with secret data ${\rm S}$ (Algorithm 2 in main document). The salt $\sigma$ guarantees that identical secret data produce different hashes and protect from pre-computation attacks and cross-share correlation.

\begin{algorithm}
	\caption{Depth Map Extraction}
	\label{alg:depthmap}
	\begin{algorithmic}[1]
		\Require Meaningful share $M$ (size $H \times W \times C$)
		\Ensure Depth map $D$ (size $H \times W \times C$)
		
		\Function{ExtractDepthMap}{$M$}
		
		\State $H, W, C = \text{dimensions}(M)$
		\State $D = \text{array of zeros}(H, W, C)$
		
		\Comment{Depth map stored in first 2 rows (fixed depth 1)}
		\For{$x = 0$ \textbf{to} $1$}
		\For{$y = 0$ \textbf{to} $W-1$}
		\For{$c = 0$ \textbf{to} $C-1$}
		\State $p = M[x,y,c]$
		\State $D[x,y,c] = p \& 0x03$ \Comment{Extract 2 bits for depth (1-3)}
		\State $D[x,y,c] = D[x,y,c] + 1$ 
		\EndFor
		\EndFor
		\EndFor
		\Comment{Rest of depth map stored in remaining rows}
		\State $\text{pos} = 2$
		\For{$x = 2$ \textbf{to} $H-1$}
		\For{$y = 0$ \textbf{to} $W-1$}
		\For{$c = 0$ \textbf{to} $C-1$}
		\If{$\text{pos} < H$}
		\State $p = M[\text{pos}, y, c]$
		\State $D[x,y,c] = p \& 0x03$
		\State $D[x,y,c] = D[x,y,c] + 1$
		\EndIf
		\EndFor
		\EndFor
		\State $\text{pos} = \text{pos} + 1$
		\EndFor
		\State \Return $D$
		\EndFunction
	\end{algorithmic}
	
\end{algorithm}

Algorithm \ref{alg:positions} generates 256 positions in each meaningful shares where hash bits are embedded. Here, the info string concatinates ``hash\_positions", the share identifier ${\rm ID}$, and the timestamp ${\rm T}$. The algorithm produces 256 bit seed by applying HKDF with the master key ${\rm K_m}$. This seed generates a cryptographically secure pseudo-random number generator (CSPRNG) produce 256 random triplets. Each triplet with modulo ${\rm H}$, ${\rm W}$ and ${\rm 3}$ calculated the coordinates ${\rm x,y,z}$. As these positions are uniformly distributed across the shares, makes it hard for an intruder to locate and corrupt the hash without the master key ${\rm K_m}$. This algorithm also used in  Cryptographic Hash Verification (Algorithm 2 in th main paper) to regenerate the same hash embedded positions.

Algorithm \ref{alg:depthmap} calculates the adaptive LSB depth map from the meaningful shares. Here, the deapth maps are stored in the deterministic positions. The algorithm first extract the depth information from first two rows. The remianing depth values are extracted from subsequent rows following a sequential layout. This algorithm is used in both Cryptographic Hash Verification (Algorithm 2) and Adaptive LSB Extraction (Algorithm 3) in the main paper.

\section{Additional Security analysis}

\begin{table}[!h]
	\caption{Dual-layer tamper detection performance against different attack types} \label{tab:td}
	\centering
		\begin{tabular}{lcc} \toprule
			\textbf{Attack Type}     & \textbf{Det. Rate} & \textbf{FPR} \\ \midrule
			Random Bit Flips (1\%)   & 99.97\%                 & 0.02\%                       \\
			Random Bit Flips (5\%)   & 100.00\%                & 0.02\%                       \\
			Cropping (10\% border)   & 98.3\%                  & 0.05\%                       \\
			Share Substitution       & 100.00\%                & 0.01\%                       \\
			JPEG Compression (Q=70)  & 95.2\%                  & 0.10\%                       \\
			Gaussian Blur ($\sigma$=1.0)    & 91.8\%                  & 0.15\%                       \\
			Adversarial Optimization & 94.5\%                  & 0.08\%             \\ \bottomrule         
	\end{tabular}
\end{table}

The proposed dual-layer tamper detection system shows strong efficacy towards various attack scenarios. Table \ref{tab:td} shows the tamper detection analysis on proposed technique in various attack types. Detection rates for bit-flipping attacks show 99.97–100.00\%, accompanied by a small 0.02\% false positive rate (FPR). This shows that the cryptographic hash layer effectively detects any bit-level alterations to the embedded secret data. Structural attacks, including share substitution, are identified with complete accuracy, whereas cropping 10\% of the image border results in a 98.3\% detection rate showing the robust spread-spectrum watermark. Even with JPEG compression (Q=70) and Gaussian blur, detection accuracy exceeds 91\%, with false positives remaining below 0.15\%. The proposed dual-layer mechanism achieves over 99.9\% for direct tampering and over 91\% for compression-based attacks, show the possibility of practical deployment.

Few significant security therms derived from the proposed technique includes,
\begin{theorem}
	\textbf{Tamper detection}: Any tampering to ${\rm E_i}$ by modifying ${\rm M_i}$ can be detected with probability $1 - 2^{-256}$.
\end{theorem}
\begin{proof}
	In the proposed technique, the embedded hash ${\rm h_i}$ is deterministic. Any tampering to ${\rm E_i}$ changes ${\rm h_i}$ and the collision probability is $2^{-256}$. When combines with watermark verification (layer 2), detection probability approaches $1$. 
\end{proof}

\begin{theorem}
	\textbf{Reconstruction fidelity}: Any shares that passes the integrity verification, the reconstructed patch ${\rm \bar{P_i}}$ will be identical to original patch ${\rm P_i}$
\end{theorem}
\begin{proof}
	As the embedding procedure is reservable, and with successful dual verification, the extracted ${\rm \bar{E_i}}$ = ${\rm E_i}$ and ${\rm \bar{AS}}$ = ${\rm AS}$. At XOR reconstruction, ${\rm \bar{P_i}}$ = ${\rm E_i}$ $\oplus$ ${\rm AS}$ = (${\rm P_i}$ $\oplus$ ${\rm AS}$) $\oplus$ ${\rm AS}$) = ${\rm P_i}$
\end{proof}

\begin{table}[!htb]
	\centering
	\caption{NPCR values (\%) for 1000 distinct shares generated per patch}
	\label{tab:npcr}
	\centering
		\begin{tabular}{cccccc}
			\hline
			\begin{tabular}[c]{@{}c@{}}\textbf{Left} \\ \textbf{eyebrow}\end{tabular} & \begin{tabular}[c]{@{}c@{}}\textbf{Right} \\ \textbf{eyebrow}\end{tabular} & \begin{tabular}[c]{@{}c@{}}\textbf{Left} \\ \textbf{eye}\end{tabular} & \begin{tabular}[c]{@{}c@{}}\textbf{Right} \\ \textbf{eye}\end{tabular} & \textbf{Nose} & \textbf{Mouth} \\ \hline 98.89  & 98.87 & 98.60  &  98.79 & 98.64   &  98.87     \\ \hline
	\end{tabular}
\end{table}

Table \ref{tab:npcr} shows Number of Pixel Change Rate (NPCR) for $1000$ unique VSS shares (meaningless) produced for each facial patch. NPCR values above 98\% shows strong encryption resilience, which calculates the pixel-wise difference between two randomly generated shares from the same patch. With NPCR values ranging from $98.60\%$ to $98.89\%$, all six patches show strong and secure encryption. These findings verify that the proposed technique offers a strong defense against statistical and differential attacks by producing highly randomized shares.

\section{Attack scenario and interception}

\begin{figure*}[!htp]
	\centering
	\includegraphics[width=0.9\linewidth]{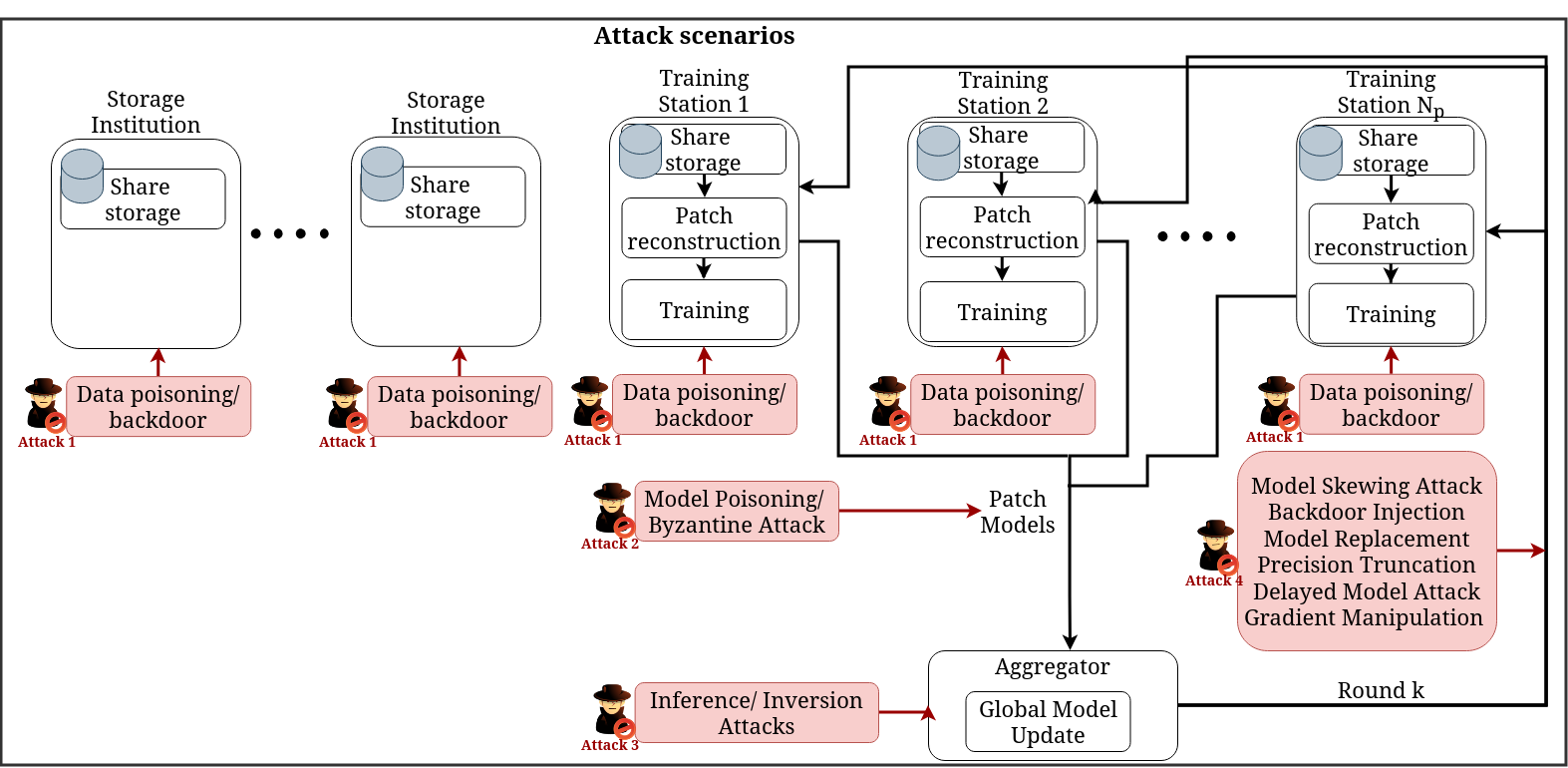}
	\caption{Attack surface of distributed multi-patch FR data storage and training}
	\label{fig:as}
\end{figure*}

Figure \ref{fig:as} shows the extensive attack scenarios of a distributed FRS, such as VOIDFace~\cite{voidface}, Federated learning~\cite{woubie2024maintaining} and proposed technique. The vulnerabilities include at storage institutions (data poisoning and backdoor injection), training workstations (Byzantine attacks, model skewing, backdoor injection, model replacement, and gradient manipulation), and the aggregator (model inversion and membership inference attacks). 

As VOIDFace lacks integrity verification, a tampered share might lead to a faulty patch reconstruction. This allows corrupted patches to get into training pipeline undetected. This can lead to local model poisoning, gradient skewness, and global model compromise through backdoor or accuracy degradation. The proposed dual layer tamper detection intercepts data poisoning before patch reconstruction. The shares failed in these verification never get into training pipeline while the verified patches contribute to training. Additionally, advanced attacks such as model skewing, backdoor injection, replacement, precision truncation, delayed attacks, and gradient manipulation are to an extend ineffective as a training node without verified shares cannot create a coherent local model. The proposed dual layer verification confirms that the training excludes only the corrupted patch rather than discarding the user sample or incorporating malicious data. Due to this, proposed technique provides a defense in depth (share-level tamper detection, selective patch exclusion during reconstruction, and broken attack chains across the distributed pipeline), as compared to VOIDFace framework.

\section{Additional performance analysis}

The proposed multi-patch-network framework comprises two principal components: the \textit{Patch Training Network (PTN)} and the \textit{Aggregator}. This framework allocates the FR task across patches and subsequently integrates their learned representations into a global embedding. The framework employs six patches thus requires six distinct \textit{PTN}, each dedicated to analyzing a particular patch. Each \textit{PTN} consists of a MobileNet backbone CNN that derives distinctive feature embeddings from the corresponding reconstructed patch. These CNNs function concurrently, allowing the model to autonomously acquire region-specific facial characteristics. The outputs of the six \textit{PTN} (512-dimensional feature vectors) are subsequently transmitted to the \textit{Aggregator}. The \textit{Aggregator} is executed as a basic fully connected layer, designed to consolidate the patch-level embeddings into a singular global feature representation of the complete face. This final embedding, also 512-dimensional, aims to learn identity-related information by integrating both local and global facial features.

\begin{figure*}[!htb]
	\centering
	\includegraphics[width=0.7\linewidth]{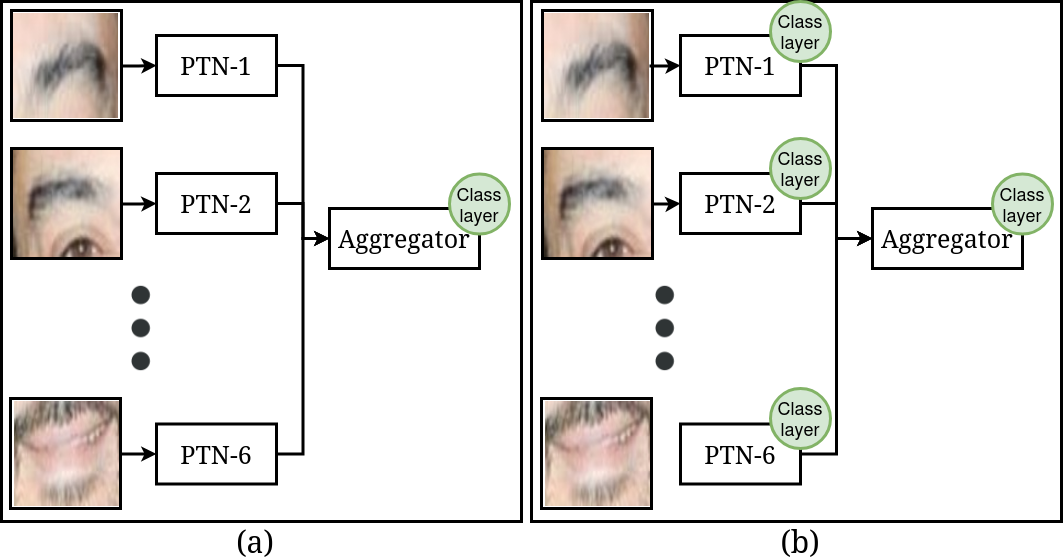}
	\caption{Multi‑Patch Training Network architectures. (a) Version 1 (V1) with supervision only on Aggregator. (b) Version 2 (V2) with supervision on both individual PTN and Aggregator.}
	\label{fig:schematics}
\end{figure*}

In the experiments, we examine two architectural variations of the multi-patch-network framework. Version 1 (V1) employs a single-task learning methodology in which only the \textit{Aggregator} generates a class prediction and \textit{PTN} function solely as feature extractors, with supervision of the final integrated embedding occurring exclusively through classification loss. Whereas, Version 2 (V2) implements a multi-task learning framework by incorporating classification heads at the output of each \textit{PTN}, overseeing both the individual patches and the \textit{Aggregator} for the identical identity classification task. This patch-level supervision enhances local discriminative feature learning and delivers more robust training results. Both V1 and V2 schematics are shown in Figure \ref{fig:schematics}.

\begin{figure}
	\centering
	\includegraphics[width=0.9\linewidth]{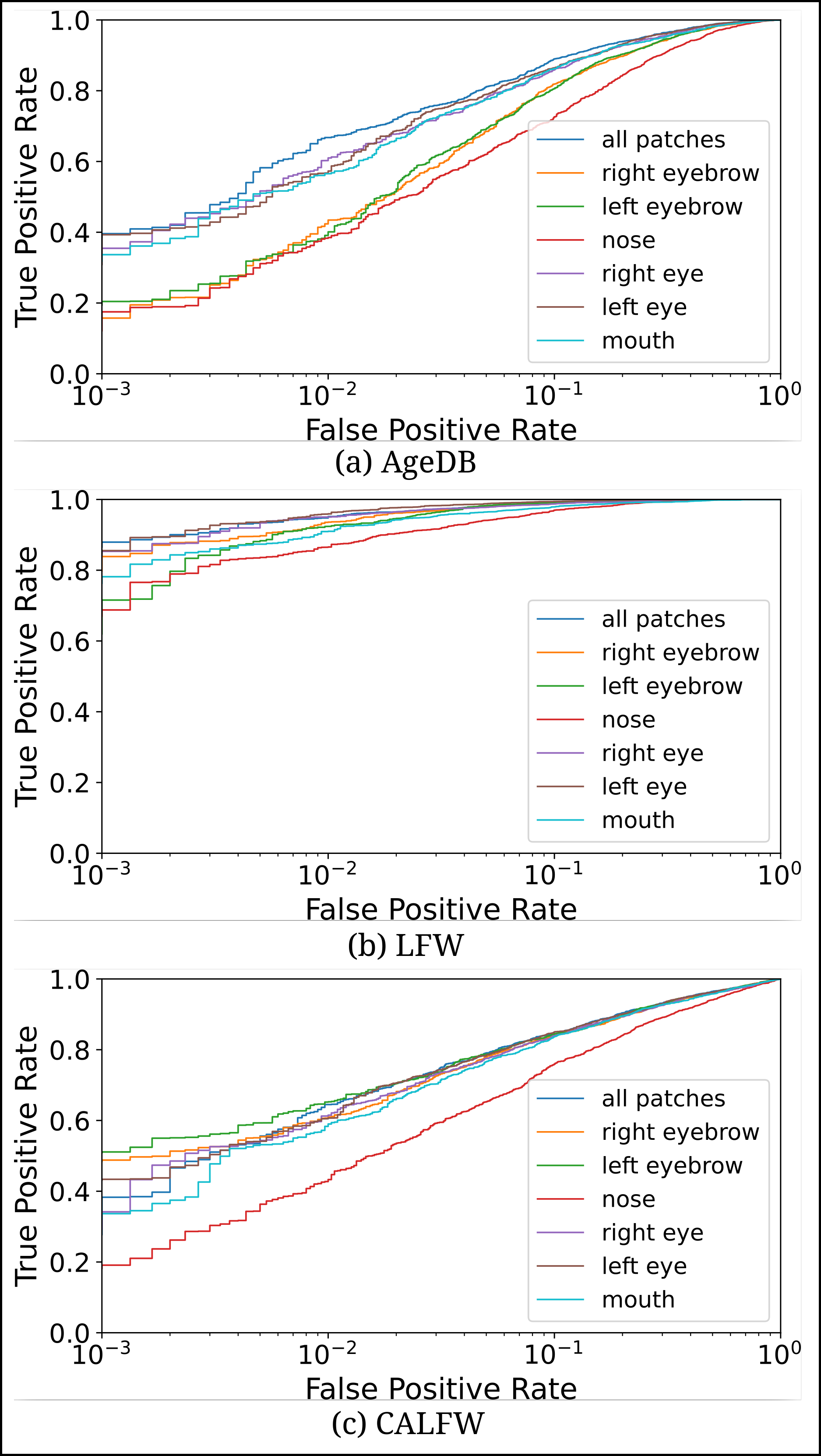}
	\caption{ROC curves of single patch models on AgeDB-30~\cite{agedb30}, LFW~\cite{LFWTechUpdate}, and CALFW~\cite{CALFW} datasets with architecture V1}
	\label{fig:sptav1}
\end{figure}

\begin{figure}
	\centering
	\includegraphics[width=0.9\linewidth]{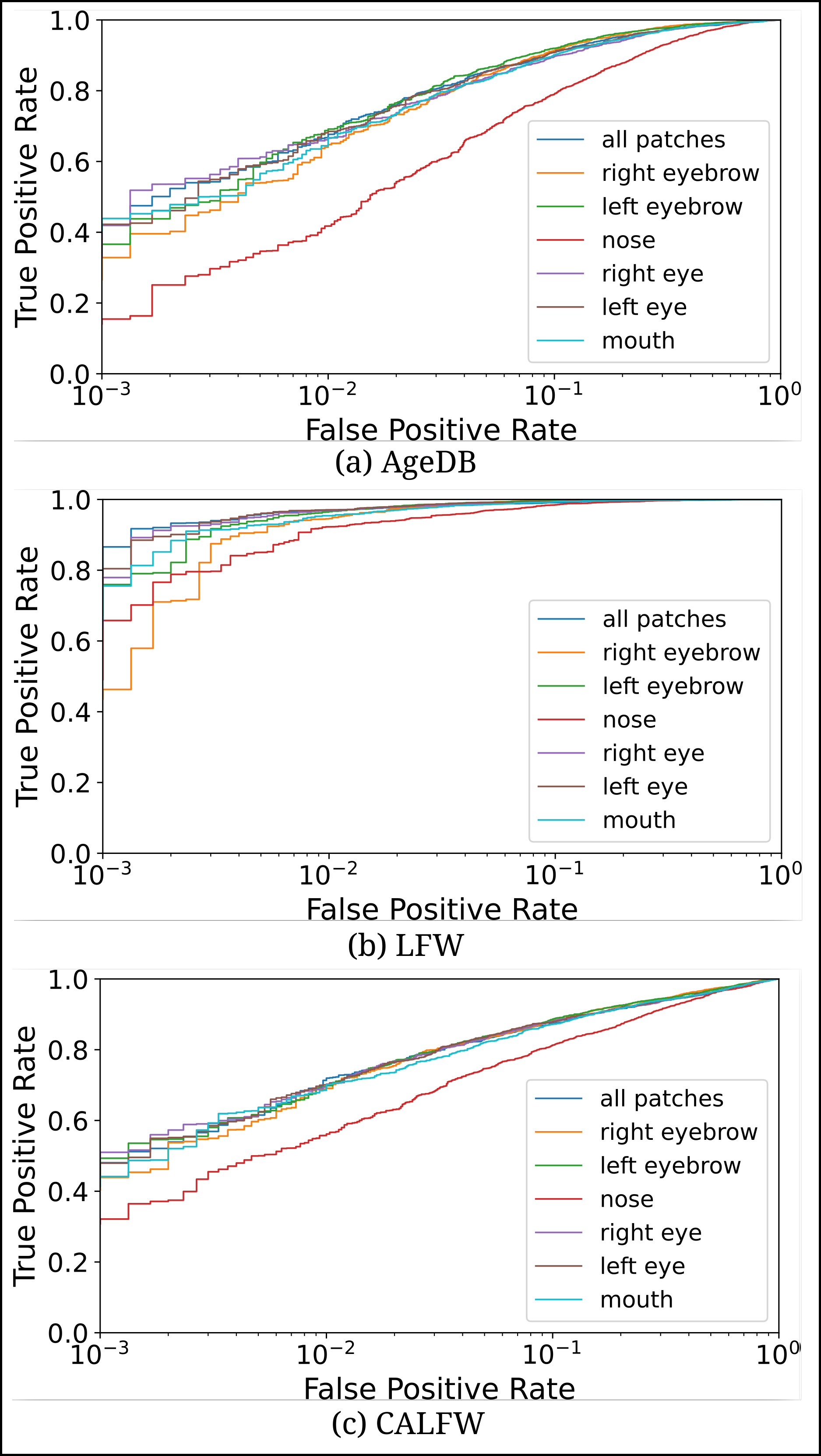}
	\caption{ROC curves of single patch models on AgeDB-30~\cite{agedb30}, LFW~\cite{LFWTechUpdate}, and CALFW~\cite{CALFW}datasets with architecture V2}
	\label{fig:sptav2}
\end{figure}

The important advantages of proposed technique in comparison with VOIDFace~\cite{voidface} is the ability to identify tampered patches and omit this for training. In VOIDFace, the patch reconstruction is performed without authentication check, and a tampered share can corrupt a patch. Using such corrupted patch for training causes backpropagation gradient errors, model accuracy loss, unstable convergence, and even targeted backdoor injection. 

In this experiment, to replicate a realistic tampering scenario, we evaluated the effects of excluding specific patches during inference. This simulate the identification and dismissal of a compromised share after integrity verification. Figure \ref{fig:sptav1} and Figure \ref{fig:sptav2} show the single patch tamper analysis on different benchmark using architecture V1 and V2, respectively. These figures shows that the proposed technique can detect the tampered share using dual layer tamper detection (cryptographic hash and robust watermark) and exclude the tampered patch from training process. This exclusion safeguards the global model from contaminated data while retaining all the other untempered face patch information. This experiment also shows an exclusive analysis on the significant of each patches and illustrates the practical advantage of the proposed technique.


\end{document}